\documentclass{article}

\usepackage{microtype}
\usepackage{graphicx}
\usepackage{subcaption}
\usepackage{booktabs} 
\usepackage{natbib}
\usepackage{hyperref}

\usepackage[preprint]{icml2026}

\usepackage{amsmath}
\usepackage{amssymb}
\usepackage{mathtools}
\usepackage{amsthm}
\usepackage{amsfonts}
\usepackage[utf8]{inputenc}

\usepackage{enumitem}
\usepackage{orcidlink}

\usepackage[capitalize,noabbrev]{cleveref}

\theoremstyle{plain}
\newtheorem{theorem}{Theorem}[section]
\newtheorem{proposition}[theorem]{Proposition}

\theoremstyle{definition}

\theoremstyle{remark}

\usepackage[textsize=tiny]{todonotes}

\usepackage[accsupp]{axessibility}  
\usepackage{multirow}
\usepackage{siunitx}
\usepackage{colortbl}

\usepackage{epsfig}
\usepackage{makecell}
\usepackage[svgnames]{xcolor}
\usepackage [english]{babel}

\usepackage{titlesec}
\titlespacing*{\section}{0pt}{0.7ex}{0.4ex}
\titlespacing*{\subsection}{0pt}{0.5ex}{0.4ex}
\titlespacing*{\subsubsection}{0pt}{0.4ex}{0.3ex}

\setlist{nosep}
\newcommand{\objectdomain}{\Omega}
\newcommand{\objectboundary}{\partial \Omega}
\newcommand{\greenfunction}{G}
\newcommand{\potential}{\varphi}
\newcommand{\greenfunctiondef}{- \frac{1}{4 \pi \|\vx - \vs \|}}
\newcommand{\chargedensity}{\rho}
\newcommand{\numsurfacesamples}{100k}
\newcommand{\predresolution}{512}

\newcommand{\laplacian}{\Delta}
\newcommand{\lossbalancingweight}{\expnumber{2}{-2}}

\newcommand{\expnumber}[2]{{#1} \times 10^{#2}}

\usepackage{amsmath,amsfonts,bm}

\def\eqref#1{equation~\ref{#1}}

\def\1{\bm{1}}

\def\vs{{\bm{s}}}

\def\vx{{\bm{x}}}

\def\mH{{\bm{H}}}

\def\mL{{\bm{L}}}

\DeclareMathAlphabet{\mathsfit}{\encodingdefault}{\sfdefault}{m}{sl}
\SetMathAlphabet{\mathsfit}{bold}{\encodingdefault}{\sfdefault}{bx}{n}

\def\gF{{\mathcal{F}}}

\def\gX{{\mathcal{X}}}
\def\gY{{\mathcal{Y}}}

\def\sR{{\mathbb{R}}}

\newcommand{\R}{\mathbb{R}}

\newcommand{\FT}[1]{ \gF \left\{ #1 \right\}}
\definecolor{bestorange}{RGB}{255,200,140}      
\definecolor{secondorange}{RGB}{255,230,180}    

\newcommand{\customsubsection}[1]{\vspace{5pt} \textbf{#1}.}

\newcommand{\topvalue}[1]{\cellcolor{bestorange}\textbf{#1}}
\newcommand{\runnerupvalue}[1]{\cellcolor{secondorange}{#1}}

\icmltitlerunning{Electrostatics-Inspired Surface Reconstruction (EISR): Recovering 3D Shapes as a Superposition of Poisson's PDE Solutions}

\begin{document}

\twocolumn[
  \icmltitle{Electrostatics-Inspired Surface Reconstruction (EISR): Recovering 3D Shapes as a Superposition of Poisson's PDE Solutions}



  \icmlsetsymbol{equal}{*}

  \begin{icmlauthorlist}
    \icmlauthor{Diego Patiño}{uta}
    \icmlauthor{Knut Peterson}{drexel}
    \icmlauthor{Kostas Daniilidis}{upenn}
    \icmlauthor{David K. Han}{drexel}
  \end{icmlauthorlist}

  \icmlaffiliation{uta}{Department of Computer Science and Engineering, University of Texas - Arlington, Arlington, TX, USA}
  \icmlaffiliation{drexel}{Department of Electrical and Computer Engineering, Drexel University, Philadelphia, PA, USA}
  \icmlaffiliation{upenn}{Department of Computer and Information Science, University of Pennsylvania, Philadelphia, PA, USA}

  \icmlcorrespondingauthor{Diego Patiño}{diego.patino@uta.edu}

  \icmlkeywords{Machine Learning, ICML}

  \vskip 0.3in
]



\printAffiliationsAndNotice{}  

\begin{abstract}
Implicit shape representation, such as SDFs, is a popular approach to recover the surface of a 3D shape as the level sets of a scalar field. Several methods approximate SDFs using machine learning strategies that exploit the knowledge that SDFs are solutions of the Eikonal partial differential equation (PDEs). In this work, we present a novel approach to surface reconstruction by encoding it as a solution to a proxy PDE, namely Poisson's equation. Then, we explore the connection between Poisson's equation and physics, e.g., the electrostatic potential due to a positive charge density. We employ Green's functions to obtain a closed-form parametric expression for the PDE's solution, and leverage the linearity of our proxy PDE to find the target shape's implicit field as a superposition of solutions. Our method shows improved results in approximating high-frequency details, even with a small number of shape priors.

\end{abstract}

\section{Introduction}

In recent years, data-driven surface reconstruction based on machine learning methods has produced a wide variety of generative methods with increasing accuracy for computing the 3D geometry of shapes from different inputs, such as point clouds or 2D images~\citep{Achlioptas2017RepresentationLA,FanSG16,Ranjan2018Generating3F,Wang2018Pixel2MeshG3,pifuSHNMKL19}. One of the methods for surface reconstruction that has gained momentum is implicit surface reconstruction (ISR)~\citep{Xu2019_NIPS,pifuSHNMKL19,Gropp2020ImplicitGR,Mescheder2019_OCCNET}. In ISR, we estimate a shape's surface as a continuous scalar field $\potential(\vx): \sR^{3} \rightarrow \sR$ to recover the geometry of a target shape $\objectboundary$ as the iso-surface of $\potential(\vx)$, such that
\begin{equation}
    \objectboundary = \left\{ \vx \in \sR^{3} \; | \; \potential(\vx) = \tau \right\}.
\end{equation}
The most common strategy for implicit surface reconstruction consists of setting $\potential(\vx)$ as the Signed Distance Function (SDF) associated with the surface of $\objectdomain$, and setting $\tau=0$ to recover the zero-crossing surface. Implicit representations often take the form of a deep neural network that approximates the SDF around the target shape~\citep{Genova2019,Michalkiewicz_2019_ICCV,OccupancyFlow,Xu2019_NIPS}. Other methods, instead, learn a decision boundary of a binary classifier~\citep{Chen2018,Mescheder2019_OCCNET,He2020GeoPIFuGA}, indicating whether a point $\vx$ is inside the surface. 

ISR has become a state-of-the-art (SotA) method since 1) it can estimate shapes of any topology (number of holes), and 2) we can sample it at arbitrary resolution to recover high-resolution meshes. However, despite their remarkable reconstruction accuracy, they have a significant drawback: they require substantial supervision. ISR methods require sufficient point-distance sample pairs $(\vx, s) \in \sR^{3}  \times \sR$ around the target shape to properly learn the SDF. This amount of supervision is infeasible in many real-world scenarios.

An alternative to the stated supervision problem comes from the known fact that an SDF is a solution of a non-linear partial differential equation (PDE) called the Eikonal equation:
\begin{equation}
\| \nabla \potential(\vx) \| = 1.
\label{eq:eikonal}
\end{equation}
Hence, one can approximate the scalar field $\potential(\vx)$ by solving Eq. \ref{eq:eikonal}. However, Eq. \ref{eq:eikonal} is challenging to solve due to its non-linear nature. Typically, SDF-based ISR methods approximate the scalar field while enforcing a PDE-loss function that regularizes the approximation to keep it close to the Eikonal equation's space of solutions. Despite their success, these methods often depend on finding an optimal trade-off between competing losses, e.g., reconstruction loss and Eikonal loss~\citep{Gropp2020ImplicitGR}. Consequently, they experience a delay in convergence, limiting the reconstruction quality achievable with a short optimization time.

The above difficulties motivate additional research on PDE-based alternatives to model scalar distance fields with similar properties, fewer supervision requirements, and potentially easier to solve. This work proposes an alternative representation of the implicit scalar field $\potential(\vx)$ by encoding it using an alternative PDE. We take inspiration from physics's electrostatic theory and replace the Eikonal equation with Poisson's equation. 

Poisson's equation is a fundamental principle in mathematical physics. It describes the distribution of electrostatic potential in a given spatial domain based on its charge densities. In 3D reconstruction, Poisson's equation serves as a powerful tool for inferring missing geometric information from sparse data points or noisy measurements. It has been used extensively in computer vision, as it arises naturally in variational problems such as rendering high dynamic range images~\citep{Fattal2002}, shadow removal~\citep{Finlayson2002}, image stitching~\citep{Agarwala2007}, image inpainting~\citep{Elder2001,Perez2003,Jia2006}, shape analysis~\citep{Kazhdan2006,Gorelick2006}, and mesh optimization~\citep{Nealen2006}. 


\begin{figure*}
\centering
\includegraphics[width=1.0\textwidth]{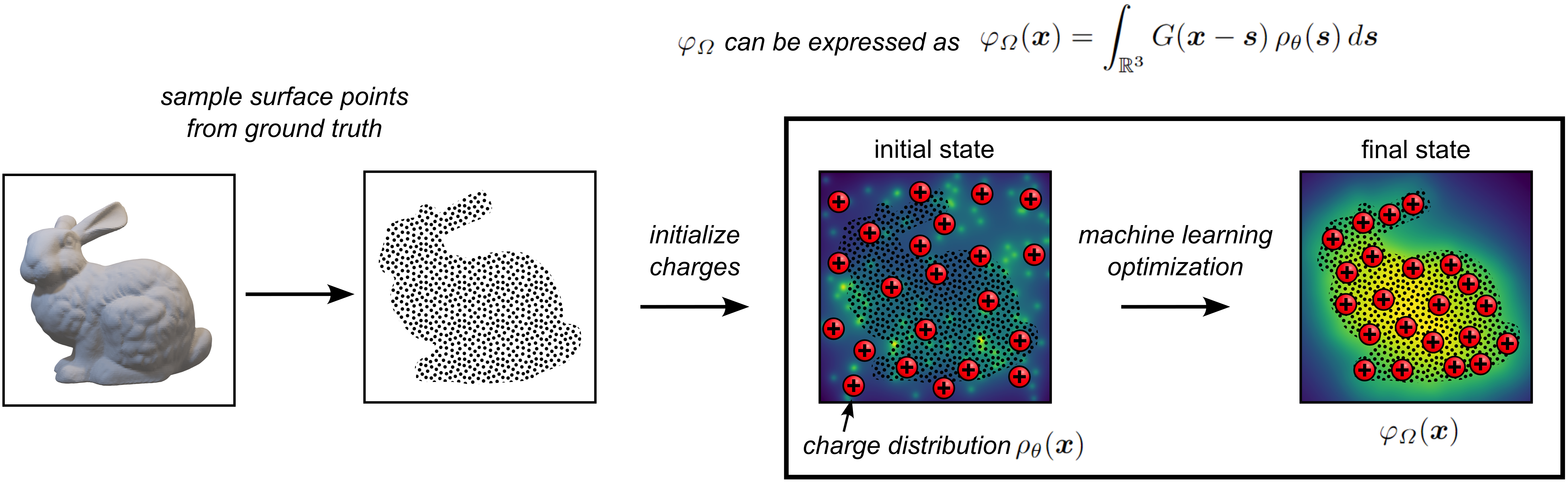}
\caption{Overview of our proposed method. We take inspiration from electrostatic theory and represent the implicit scalar field to reconstruct a target shape as a solution of Poisson's equation. We solve the PDE analytically using Green's functions and parametrize the charge density function to fit the desired shape.}
\label{fig:overview}
\end{figure*}

In this work, we propose reconstructing a target shape $\objectdomain$ by finding an implicit scalar field that is the electrostatic potential of an everywhere-positive charge distribution within it. We adopt Poisson's equation as a substitute for the Eikonal equation and solve it analytically using Green's functions. Through careful derivation, we show how to express a solution of Poisson's PDE as the convolution of the Green function and a parametric charge density function $\chargedensity(\vx)$. Following our derivation, we find a closed-form expression for the PDE's solution in terms of the charge density, which we parameterize as a series of isotropic Gaussian \textit{charge} densities -- shape priors -- to recover the surface of the desired shape. By leveraging the linearity of the Poisson equation, we represent the implicit field of a target shape as the superposition of the electrostatic potential of the above-mentioned shape priors. Finally, we adopt a machine learning approach to optimize the charge parameters. Fig. \ref{fig:overview} briefly summarizes our method's working principle. The contributions of our surface representation method are threefold:

\begin{itemize}[left=1pt]
    \item We find a proxy PDE as a replacement for the Eikonal equation, avoiding costly loss functions when approximating the scalar field in implicit shape representation. 
    \item We leverage the linearity of our proxy PDE to solve it analytically using Green's functions and to express the target shape as a superposition of solutions to the PDE.
    \item We demonstrate the capabilities of our shape representation to approximate high-frequency details of the target shape with a relatively small number of parameters.
\end{itemize}

\section{Related Work}

\customsubsection{Deep Learning-based Surface Reconstruction} 

Deep learning has enabled learning-based representations to reconstruct 3D shapes from point clouds or images~\citep{Achlioptas2017RepresentationLA,FanSG16,Ranjan2018Generating3F,Wang2018Pixel2MeshG3,pifuSHNMKL19}. The most popular class of these methods is implicit surface reconstruction (ISR)~\citep{Xu2019_NIPS,pifuSHNMKL19,Gropp2020ImplicitGR,Mescheder2019_OCCNET}. They are widely utilized due to two properties: they can estimate shapes of any topology  -- number of holes --  and can be sampled at arbitrary resolution to produce high-resolution meshes. 

In ISR, a 3D shape is represented as a level set of a continuous scalar field, often encoded using a coordinate-based neural network (NN). One can extract an artifact-free triangular mesh of the target shape from this field via an iso-surface extraction algorithm, e.g., Marching Cubes. ISR models predict SDFs~\citep{Atzmon2019SALSA,Gropp2020ImplicitGR,Michalkiewicz_2019_ICCV,Park2019,Takikawa2021NeuralGL} or occupancy probabilities~\citep{Chen2018,OccupancyFlow,Mescheder2019_OCCNET} from partial observations.

Most recently, neural radiance fields (NeRF) ~\citep{Mildenhall2020NeRFRS} have advanced 3D scene representation by encoding a scene's radiance field with deep NNs and then employing volumetric rendering to enable novel-view synthesis. Although NeRF has become a ubiquitous method for both 3D reconstruction and generation ~\citep{Barron2021MipNeRF3U,Li2023NeuralangeloHN,chen2022mobilenerf,Poole2022DreamFusionTU,Chan2021EfficientG3}, optimizing NeRF can be time-consuming. Moreover, recovering the surface from NeRFs is not straightforward and often requires cumbersome strategies to enforce surface reconstruction via SDFs and the Eikonal equation.

\customsubsection{Reconstructing 3D Objects with Geometric Priors} 
Implicit function-based surface representations encode shape priors in neural networks and latent embeddings~\citep{Genova2019LocalDI,Yavartanoo20213DIAS3S}. However, they struggle with high-frequency details and novel classes~\citep{Sitzmann2020ImplicitNR,Genova2019LocalDI}. Previous work uses primitive-based representations to approximate complex shapes as combinations of simple primitives, e.g., cubes~\citep{Tulsiani2016LearningSA}, ellipsoids~\citep{Genova2019}, superquadrics~\citep{Paschalidou2019SuperquadricsRL}, convex polytopes~\citep{Deng2019CvxNetLC}, or polynomial zero-level sets~\citep{Yavartanoo20213DIAS3S}. However, combining primitives into a unified implicit function remains unclear. Consequently, these methods often introduce additional loss functions and refining networks to produce a coherent implicit field for surface reconstruction~\citep{Genova2019LocalDI,Yavartanoo20213DIAS3S,chen2023textto3d,Tang2023DreamGaussianGG,yi2023gaussiandreamer}. Moreover, primitives are increasingly being used for articulated shapes~\citep{Lei2023GARTGA,Li2024GaussianBodyCH,Yuan2023GAvatarA3} and novel-view rendering with Gaussian splatting~\citep{kerbl3Dgaussians,Keselman2022ApproximateDR,luiten2023dynamic,yang2023deformable3dgs}. However, these methods require many shape priors for achieving high-fidelity 3D reconstructions. 

\customsubsection{Physics-Informed Machine Learning (PIML)} Deep learning methods have enabled high-accuracy surface reconstruction. However, current methods rely on complex architectures trained on large datasets. Recently, PIML has emerged as a potential solution to reduce the complexity of such architectures~\citep{Raissi2019PINN,Banerjee2023PhysicsInformedCV, Kashinath2021,meng2022physics}. The basis of PIML is to incorporates physics models and observational data to guide learning, improving robustness, accuracy, and convergence~\citep{Karniadakis2021PhysicsinformedML,meng2022physics,Cuomo2022,Hao2022PhysicsInformedML,Kashinath2021}. The introduction of physical information into machine learning frameworks encoded as partial differential equations (PDEs), symmetry constraints, or conservation laws has started opening and transforming many application domains~\citep{Hao2022PhysicsInformedML,Cranmer2020LagrangianNN}. Examples of these methods include fluid dynamics~\citep{Brunton2020MLforFluidMechanics}, physics discovery~\citep{Champion2020,Baddoo2021PhysicsinformedDM}, NeuralODEs~\citep{Chen2018NeuralOD}, human pose estimation, and geometry-aware 3D reconstruction~\citep{Gartner2022DifferentiableDF,Grtner2022TrajectoryOF}, geometry aware 3D reconstruction using NeRFs~\citep{Kosiorek2021NeRFVAEAG} or SDFs~\citep{Gropp2020ImplicitGR,Marschner2023}. These methods use PDEs, ODEs, or physics laws as regularizers or inductive biases.

In this work, we study the problem of reconstructing 3D objects with geometric priors from a physically inspired approach based on the well-studied phenomenon of electrostatic potential. We aim to reconstruct an object's geometry by assuming its surface is an isosurface of the electric potential generated by a distribution of positive charges strategically placed inside the object. We leverage the linearity of the PDE governing the electric potential to represent the target shape as a linear combination of solutions to the PDE for individual Gaussian charge densities. 
\section{Proposed Method}

Consider a bounded domain $\objectdomain \in \sR^{3}$, along with its associated scalar field $\potential(\vx): \sR^{3} \mapsto \sR$ implicitly representing the target 3D shape as one of its iso-surfaces. Similar to ~\citep{Gorelick2006,Aubert2014}, we approximate $\potential(\vx)$ as a solution of a PDE. However, we draw inspiration from physics' \textit{electrostatic theory} and replace the familiar Eikonal PDE with Poisson's equation. We then solve the PDE analytically and learn a parameterization of the charge density via a machine-learning approach to reconstruct the target's surface.

\subsection{Electrostatic Potential-based SDF}
A popular choice in the surface reconstruction SotA is to model a scalar field $\potential(\vx)$ encoding the shape as an SDF while explicitly enforcing it to be a solution of the Eikonal partial differential equation using regularization loss functions. Inspired by previous work~\citep{Gorelick2006,Aubert2014}, we propose an alternative model for $\potential(\vx)$ as the electrostatic potential of a positive charge density governed by Poisson's PDE:
\begin{equation}
    \Delta \potential(\vx) = - \frac{\chargedensity(\vx)}{\epsilon_{0}},
    \label{eq:poisson_eq}    
\end{equation}
conditioned to satisfy
\begin{align}    
    \rho(\vx) &= \tau , \quad \text{for }  \vx \in \objectboundary \nonumber \\ 
    \chargedensity(\vx)  &> 0 , \quad \text{for } \vx \in \sR^{3}. \label{eq:poisson_eq_conditions}    
\end{align}
In Eq. \ref{eq:poisson_eq} and \ref{eq:poisson_eq_conditions}, the differential operator $\laplacian$ stands for the Laplacian, $\chargedensity(\vx)$ is the charge density at each $\vx \in \R^{3}$, $\tau$ is a constant positive value, and $\epsilon_{0}$ is the permittivity of the medium. In this work, we set $\epsilon_{0} = 1$ without any loss of generality. Additionally, it is worth mentioning that solutions to Eq. \ref{eq:poisson_eq} can be proven to be always unique and smooth given just Dirichlet boundary conditions, in contrast with the Eikonal equation.

Before attempting to approximate $\potential$, we need to show that Eq. \ref{eq:poisson_eq} is a suitable alternative to replace the SDF of $\objectdomain$, at least for the task of reconstructing its boundary via iso-surface extraction. To this end, $\potential(\vx)$ needs to satisfy two requirements that solutions of the Eikonal equation satisfy, too. First, $\potential(\vx)$ must have the same constant value $\tau$ at all points on $\objectboundary$ -- the desired iso-surface. Second, we require $\potential(\vx)$ to be positive on all points inside the domain $\objectdomain$ and negative outside, as with SDFs. This second requirement is essential when recovering the surface as a triangular mesh using an iso-surface extraction algorithm such as Marching Cubes~\cite{Lorensen1987}. 

Our electrostatics-inspired approach satisfies the first requirement because it is subject to the boundary conditions in Eq. \ref{eq:poisson_eq_conditions}. Satisfying the second requirement arises from a well-known property of Poisson's equation solutions, \textit{the minimum principle}. The minimum principle states that if $\potential(\vx)$ is a solution of Eq. \ref{eq:poisson_eq}, and $\objectboundary$ is the boundary of a closed region, then if $$- \frac{\chargedensity(\vx)}{\epsilon_{0}} < 0, \; \forall \vx \in \objectdomain,$$ $\potential(\vx)$ takes its minimum value in $\objectboundary$ and not in its interior. Hence, all scalar field values inside a level-set region are greater than the values on the level-set. We provide a short mathematical proof of the minimum principle in Appx. \ref{appx:min_principle}. Note that by leveraging the linearity of Poisson's equation, we enable the use of useful mathematical tools to solve the PDE. These tools were previously unavailable for modeling $\potential(\vx)$ when using the Eikonal equation.

\subsection{Solving Poisson Equation with Green's Functions}
Our goal is to find a scalar field $\potential: \sR^{3} \mapsto \sR$ that implicitly represents the geometry of a target shape $\objectdomain$, encoded as a solution of the electrostatic potential of an always-positive charge density $\chargedensity(\vx)$ as per Eq. \ref{eq:poisson_eq}, subject to the corresponding boundary conditions in Eq. \ref{eq:poisson_eq_conditions}. Because $\laplacian$ is a linear operator, the PDE accepts solutions through the use of Green's functions. Green's functions are the impulse response of a non-homogeneous linear differential operators. We can interpret Green's functions as an integral kernel that is the inverse of the operator~\citep{duffy2001green} and thus can be used to find $\potential(\vx)$ by solving first the associated PDE
\begin{equation}
\laplacian G(\vx, \vs) = \delta (\vx - \vs),
\label{eq:green_function_pde}
\end{equation}
where $\delta$ is the Dirac delta function translated by $\vs$ and $G$ is the Green's function for the Laplacian. 

Since $\laplacian$ is translation-invariant, we can write $G$ as a convolution kernel such that $G(\vx, \vs)= G(\vx - \vs)$. Following this result, we multiply Eq. \ref{eq:green_function_pde} by $\chargedensity(\vs)$ on both sides, and then integrate over $\sR^{3}$ with respect to $\vs$. Next, we take advantage of the Laplacian's linearity and interchange $\laplacian$ with the integral sign to obtain
\begin{align}
\laplacian G(\vx - \vs) &= \delta(\vx - \vs) \nonumber \\
\int_{\sR^{3}} \laplacian G(\vx - \vs) \chargedensity(\vs) d\vs &= \int_{\sR^{3}} \delta (\vx - \vs) \chargedensity(\vs) d\vs. \nonumber \\
\laplacian \left[ \int_{\sR^{3}} G(\vx - \vs) \chargedensity(\vs) d\vs \right] &= \int_{\sR^{3}} \delta (\vx - \vs) \chargedensity(\vs) d\vs. \nonumber \\
\laplacian \left[ \left( G \ast \chargedensity \right) (\vx) \right] &= \chargedensity(\vx) \label{eq:conv_solution}.
\end{align}
Consequently, we can write the solution to the PDE in Eq. \ref{eq:poisson_eq} -- and thus find our target scalar field -- as the convolution of the Green's function with the charge density, $\chargedensity(\vx)$.

However, we are not interested in finding just any $\potential(\vx)$, but one that can act as an implicit scalar field for reconstructing $\objectdomain$. To this end, we construct a parametric approximation of the charge density, e.g., parametrized with a neural network. Later, we write the implicit scalar field following Eq. \ref{eq:conv_solution} as 
\begin{equation}
\potential_{\objectdomain}(\vx) = - \frac{1}{\epsilon_{0}} \int_{\sR^{3}} G(\vx - \vs)\,\chargedensity_{\theta}(\vs)\,d\vs,
\label{eq:potential_solution}
\end{equation}
where $\potential_{\objectdomain}$ is a solution of our target PDE for a charge density $\chargedensity_{\theta}(\vs)$ parameterized with $\theta$, and boundary conditions $\potential_{\objectdomain}(\vx) = \tau$ along $\objectboundary$ for a constant value $\tau > 0$.

\subsection{Surface Reconstruction as a Linear Combination of Gaussian Charges}
Because Poisson's equation has been extensively studied, the Green's function for the Laplacian is a known result of the form
\begin{equation}
   \greenfunction(\vx, \vs) = \greenfunctiondef.
   \label{eq:laplacian_greens_function}
\end{equation}
However, evaluating $\potential$ at a point $\vx$ is challenging, as it requires computing the integral each time, which is neither straightforward nor efficient. For an arbitrary parametrization of $\chargedensity(\vx)$, the computation of $\potential$ becomes infeasible because we need to evaluate it on each point in a 3-dimensional fine-grained grid in order to use the Marching Cubes algorithm for surface reconstruction.  

To overcome the above problem, let us consider an isotropic Gaussian-shaped charge density 
\begin{equation}
\displaystyle \chargedensity _{\theta}(\vx)={\frac {Q}{\sigma ^{3}{\sqrt {2\pi }}^{3}}}\,\exp \left( - \frac{  (\vx - \vs)^{\intercal}(\vx - \vs) }{2\sigma ^{2}} \right),
\label{eq:guassian_charge_dist}
\end{equation}
where $\vs$ denotes the charge location, $Q$ is a positive charge magnitude, and $\sigma$ is the charge density spread. By adopting this parameterization of the charge density, we can work out the integral in Eq. \ref{eq:potential_solution} to find a simplified expression for the electrostatic potential of a single Gaussian-distributed charge as
\begin{equation}
    \potential(\vx)={\frac {1}{4\pi \varepsilon_0 }}{\frac {Q}{\| \vx - \vs \|}}\operatorname{erf} \left({\frac {\| \vx - \vs \|}{{\sqrt {2}}\sigma }}\right),
    \label{eq:shape_prior_solution}
\end{equation}
with $$\displaystyle \operatorname {erf} (z) ={\frac {2}{\sqrt {\pi }}}\int _{0}^{z}e^{-t^{2}}\,\mathrm {d} t.$$
The reader can explicitly check the validity of Eq. \ref{eq:shape_prior_solution} as a solution of the target PDE by computing $\laplacian \potential(\vx)$. 

We have found a parametric formula in Eq. \ref{eq:shape_prior_solution} for an implicit scalar field whose iso-surfaces are spheres. However, we are concerned with representing 3D shapes that are likely far more complicated. To achieve this goal, we leverage the linearity of Poisson's equation to seamlessly combine $K$ shape priors into a superposition of solutions of the PDE. We parametrize each prior with $\theta_{i} = \left\{ \vs_{i}, Q_{i}, \sigma_{i} \right\}$, such that
{\small
\begin{equation}
    \potential_{\objectdomain}(\vx) = \sum_{i=1}^{K} \potential_{i}(\vx) = \sum_{i=1}^{K} {\frac {1}{4\pi \varepsilon_0 }}{\frac {Q_{i}}{\| \vx - \vs_{i} \!\|}}\operatorname {erf} \left({\frac {\| \vx - \vs_{i}\!\|}{{\sqrt {2}}\sigma_{i} }}\right).
    \label{eq:potential_of_the_object_object}
\end{equation}}
The last step in our method is to estimate $K$ Gaussian charge distributions to approximate the target shape $\objectdomain$. 

\subsection{Learning the Charge Distribution}
To reconstruct the surface with our method, we optimize the charge parameters to fit the target shape $\objectdomain$. To this end, we guide our method with two loss functions: a boundary-condition loss and a charge-restriction loss.

Let $\gX \subset \objectboundary$ be a set of points sampled from the target object's surface. We define the boundary condition loss as
\begin{equation}
   \mathcal{L}_{bc} = \frac{1}{|\gX|} \sum _{\vx \in \gX} \| \potential_{\objectdomain}(\vx) - \tau \|^{2},\, \tau > 0.
\end{equation}
We use $\mathcal{L}_{bc}$ to enforce $\potential_{\objectdomain}(\vx)$ to have an iso-surface at a value $\tau$ on all points belonging to $\objectboundary$. 

Additionally, recall that one fundamental property of the desired implicit field is that $\potential_{\objectdomain}(\vx) > \tau, \; \forall \vx \in \objectdomain$. This property ensures the feasibility of Marching Cubes in recovering the object's shape as a triangular mesh. Unfortunately, this property is not fully satisfied when the charges are outside $\objectdomain$. To avoid this situation, we define the charge restriction loss to enforce the charges to be inside the shape, such that
\begin{equation}
   \mathcal{L}_{cr} = \frac{1}{K} \sum _{i = 1}^{K} d_{\operatorname{min}}\left( \vs_{i}, \gY \right)^{2},
\end{equation}
where $d_{\operatorname{min}}\left( \vs_{i}, \gY \right)$ is the minimum distance between a charge $\vs_{i}$ and a set of points $\gY \subset \objectdomain$. Finally, we train our method by jointly optimizing $\mathcal{L}_{bc}$ and $\mathcal{L}_{cr}$ with a balance weight set experimentally to $\lambda = \lossbalancingweight$. We define the total loss function as
\begin{equation}
    \mathcal{L} = \mathcal{L}_{bc} + \lambda \mathcal{L}_{cr}.
\end{equation}

Notice that, in contrast with other methods in the SotA, we do not require enforcing a PDE regularization loss, since we explicitly compute its solution. Hence, our formulation allows the loss functions to focus solely on recovering the target shape's geometry.

\section{Experiments}

\subsection{Experimental Setup}

\customsubsection{Datasets and Baselines}
We compare our approach with the Implicit Geometric Regularization (IGR) method~\citep{Gropp2020ImplicitGR}, and with the classical Poisson Reconstruction~\cite{Kazhdan2006}. We use the dataset employed in IGR provided by ~\citep{Williams2018DeepGP}. Additionally, we evaluate on a subset of scanned objects from the classical Stanford 3D Scanning Repository. These datasets comprise triangular watertight meshes computed from registered point clouds acquired from multiple views of real-world objects. 

\customsubsection{Evaluation Metrics}
We follow standard metrics in surface reconstruction ~\citep{Park2019,Mescheder2019_OCCNET,yariv2021volume,Yu2022MonoSDF,Tatarchenko2019WhatDS,Knapitsch2017TanksAT}, and report Chamfer distance ($d_C$), F1 score, normal consistency (NC), the Hausdorff distance ($d_H$), and intersection over union (IoU). To compute the metrics, we sample $\numsurfacesamples$ points and their normals on the surfaces of both the predicted and ground-truth meshes. For the IoU, we rasterize the predictions and the ground-truth mesh at a resolution of $\predresolution$.

\customsubsection{Implementation Details}
We implement our code in \texttt{PyTorch} using Adam~\citep{Kingma2014AdamAM} to optimize each object's charges for $60k$ steps using cosine-annealed learning rates from $lr = \expnumber{1}{-3}$ to $lr=\expnumber{1}{-7}$. Before optimizing the charges, we randomly initialize their parameters such that in the first iteration, $Q_{i} = \expnumber{1}{-7}$, the charge locations $\vs_{i}$ are uniformly distributed in the interval $[-0.5, 0.5]^{3}$, and the charge spread $\sigma_{i}$ follow a normal distribution centered a zero. To compute the loss functions, we extract $250k$ points from the ground truth surface and sample $16k$ of them at each training iteration to optimize the loss. We balance the boundary condition loss and the charge restriction loss by setting the $\lambda = \lossbalancingweight$. We optimize all models on a single NVIDIA H10 GPU. Optimizing each model takes approximately $8$ minutes for $5k$ charges and $30$ minutes for $50k$ charges. We set the threshold to $\tau = 1.0$ in all our experiments. After review and upon acceptance, we will make the source code available.

\subsection{Surface Reconstruction Accuracy}

We experiment with varying the number of charges during the optimization stage and analyze its effect on the reconstructed surface quality. We summarize our results in Fig. \ref{fig:varying_num_charges} and Tab. \ref{tab:varying_num_charges}. We report the intersection over union, absolute normal consistency, and the F1 score with a threshold of $1\%$ on the shape's maximum side length~\citep{Tatarchenko2019WhatDS}.

\begin{figure*}
\centering
\includegraphics[width=0.9\textwidth]{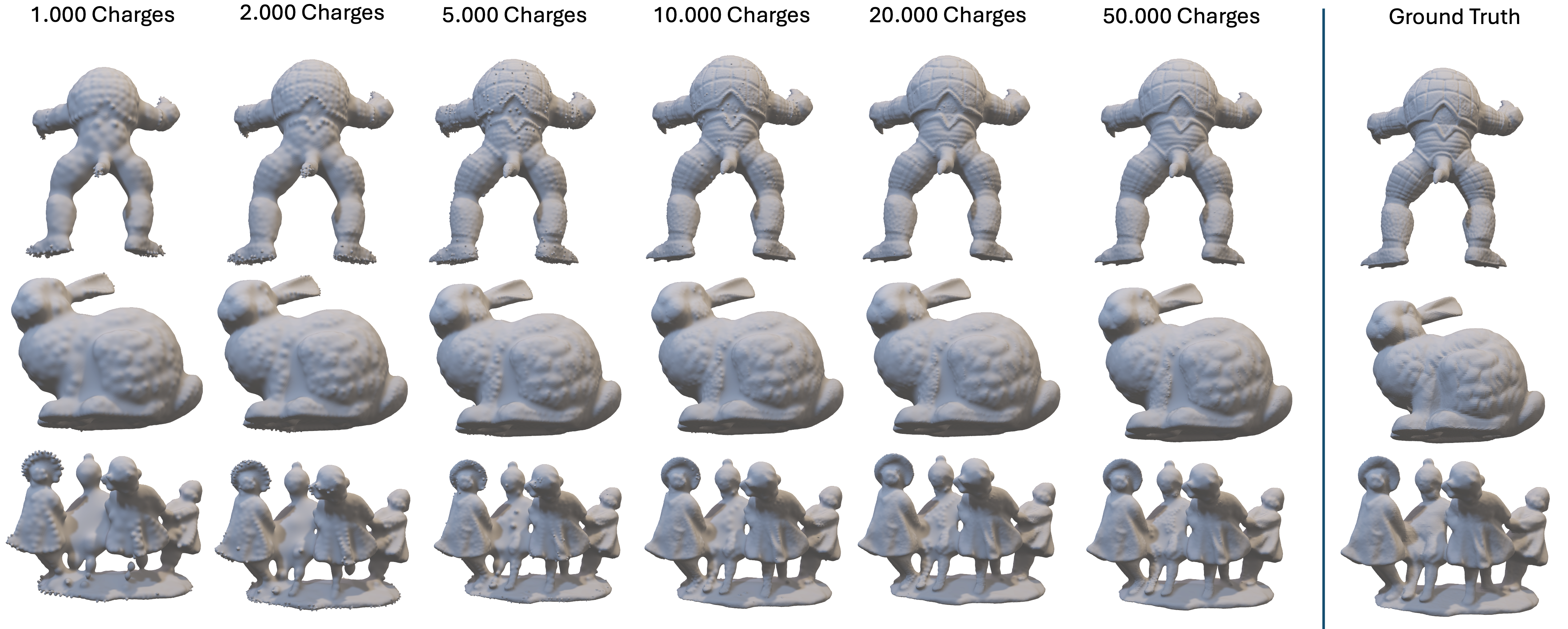}
\caption{Surface reconstruction vs. number of charges. Our method achieves accurate results with only $1k$ charges when reconstructing a shape's overall geometry. Increasing the charges significantly improves the reconstruction's high-frequency details.}
\label{fig:varying_num_charges} 
\end{figure*}

\begin{table*}
    \centering
    \setlength{\tabcolsep}{1.9pt}
    \caption{Surface reconstruction with different numbers of Gaussian charges (shape priors) on a subset of the Stanford 3D Scanning repository. Our method shows significant results for all the metrics, even at a lower number of shape priors. See full table in Appx.~\ref{apx:varying_gaussians}.}
    \label{tab:varying_num_charges}
    \scriptsize
    \begin{tabular}{lccccccccccccccccccc}
        \toprule
        \textbf{Model}  & \multicolumn{3}{c}{\textbf{1000 Charges}} & \multicolumn{3}{c}{\textbf{2000 Charges}} & \multicolumn{3}{c}{\textbf{5000 Charges}} & \multicolumn{3}{c}{\textbf{10000 Charges}} & \multicolumn{3}{c}{\textbf{20000 Charges}} & \multicolumn{3}{c}{\textbf{50000 Charges}} & \textbf{\# verts.}\\
        \cmidrule(lr){2-4}\cmidrule(lr){5-7}\cmidrule(lr){8-10}\cmidrule(lr){11-13}\cmidrule(lr){14-16}\cmidrule(lr){17-19}
         & F1($\uparrow$) & IoU($\uparrow$) & NC($\uparrow$) & F1($\uparrow$) & IoU($\uparrow$) & NC($\uparrow$) & F1($\uparrow$) & IoU($\uparrow$) & NC($\uparrow$) & F1($\uparrow$) & IoU($\uparrow$) & NC($\uparrow$) & F1($\uparrow$) & IoU($\uparrow$) & NC($\uparrow$) & F1($\uparrow$) & IoU($\uparrow$) & NC($\uparrow$) & \\
        \midrule                
        Armadillo & 98.09 & 0.93 & 0.99 & 99.24 & 0.94 & 0.99 & 99.65 & 0.95 & 0.99 & 99.78 & 0.97 & 1.00 & 99.82 & 0.98 & 1.00 & 99.85 & 0.98 & 1.00 & $\expnumber{1.7}{5}$  \\
        Happy Buddha & 91.28 & 0.85 & 0.97 & 93.78 & 0.85 & 0.97 & 95.99 & 0.86 & 0.98 & 96.79 & 0.91 & 0.99 & 97.05 & 0.94 & 0.99 & 97.27 & 0.95 & 0.99 & $\expnumber{5.4}{5}$  \\        
        Asian dragon & 90.84 & 0.80 & 0.96 & 92.04 & 0.77 & 0.96 & 95.24 & 0.76 & 0.97 & 96.24 & 0.76 & 0.97 & 96.46 & 0.82 & 0.98 & 96.83 & 0.86 & 0.98 & $\expnumber{1.9}{7}$  \\
        \bottomrule
    \end{tabular}
\end{table*}

Notice how EISR effectively reconstructs high-frequency details, i.e., the folds on the patterns on the Armadillo's back, and the nose details in the Bunny model. These results demonstrate our method's lightweight representation capabilities, since we achieve good approximations of the overall shapes even with only $5k$ Gaussian charges. With our method, we can progressively obtain more accurate details at a small number of charges relative to the number of vertices in the ground-truth mesh. Note that some reconstructions exhibit bump and pit artifacts. These artifacts occur at low Gaussian counts ($\leq 20000$) because the sparse representation cannot capture fine surface details. However, increasing the number of charges improves surface quality by reducing these irregularities, thereby demonstrating the effectiveness of our shape representation.

Figure~\ref{fig:cross_section} shows a cross-section of the scalar field produced by our method, illustrating the smooth signed-distance structure induced by the learned representation. The zero-level set aligns consistently with the target surface, confirming that our approach recovers a well-behaved implicit surface representation.

\subsection{Comparative Results}

We compared our approach with IGR and Poisson reconstruction, achieving competitive results. Because Poisson reconstruction strongly depends on surface normals, we used two versions of the baseline: ground-truth normals and estimated normals. We compute Poisson reconstructions using Open3D\footnote{Open3d: \url{https://www.open3d.org}} and libigl\footnote{libigl: \url{https://libigl.github.io}} library. We use a simple PCA strategy implemented in Open3D to estimate the normals. While we are aware that more accurate methods exist, we consider them post-processing and therefore exclude them to ensure a fair comparison. Note that our method does not require normal data, since the surface-direction regularization is built into the PDE formulation. We report the Chamfer distance ($d_C$) and the Hausdorff distance ($d_H$). We summarize our results in Fig. \ref{fig:results_stanford_3d} and Tab. \ref{tab:standford_3d_comp}.


\begin{table*}
    \centering
    \setlength{\tabcolsep}{2.8pt}
    \footnotesize    
    \caption{Reconstruction results. We compare our method against IGR~\cite{Gropp2020ImplicitGR} and the classic Poisson reconstruction~\cite{Kazhdan2006}. We report the Chamfer distance $d_{C}$ multiplied by a $\expnumber{1}{2}$ factor, and the Hausdorff distance $d_{H}$ by $\expnumber{1}{1}$. The symbol $^\dagger$ indicates that we use ground-truth surface normals for Poisson reconstruction; otherwise, we estimate them from the input point cloud. We highlight the best and second-best values. Lower values are better.}
    \label{tab:standford_3d_comp}
    \begin{tabular}{lcccccccccc|cccccccccccccccccc}
    \toprule
     & \multicolumn{10}{c}{Stanford 3D Repository} & \multicolumn{10}{c}{Williams et al. dataset} \\
     \textbf{Method} & \multicolumn{2}{c}{\textbf{Bunny}} & \multicolumn{2}{c}{\textbf{Dragon}} & \multicolumn{2}{c}{\textbf{Armadillo}} & \multicolumn{2}{c}{\textbf{Happy B.}} & \multicolumn{2}{c}{\textbf{T. Statue}} & \multicolumn{2}{c}{\textbf{Anchor}} & \multicolumn{2}{c}{\textbf{Daratech}} & \multicolumn{2}{c}{\textbf{Dc}} & \multicolumn{2}{c}{\textbf{Gargoyle}} & \multicolumn{2}{c}{\textbf{Lord Quas}}\\
     & $d_{C}$ & $d_{H}$ & $d_{C}$ & $d_{H}$ & $d_{C}$ & $d_{H}$ & $d_{C}$ & $d_{H}$ & $d_{C}$ & $d_{H}$ & $d_{C}$ & $d_{H}$ & $d_{C}$ & $d_{H}$ & $d_{C}$ & $d_{H}$ & $d_{C}$ & $d_{H}$ & $d_{C}$ & $d_{H}$ \\
        \toprule
         IGR &
         \runnerupvalue{0.53} & 0.53 &
         \runnerupvalue{0.35} & \topvalue{0.13} &
         \runnerupvalue{0.35} & \topvalue{0.06} &
         \topvalue{0.32} & \runnerupvalue{0.10} &
         \runnerupvalue{0.47} & 0.79 &
         1.30 & 1.08 &
         1.72 & 1.13 &
         \runnerupvalue{0.61} & \runnerupvalue{0.36} &
         1.32 & 1.67 &
         0.61 & 0.47 \\

         Poisson rec. &
         1.41 & \runnerupvalue{0.43} &
         1.44 & 0.39 &
         1.49 & 0.45 &
         1.09 & 0.34 &
         0.98 & \runnerupvalue{0.26} &
         1.25 & \runnerupvalue{0.42} &
         0.99 & \runnerupvalue{0.25} &
         1.48 & 1.59 &
         1.56 & \runnerupvalue{0.39} &
         1.22 & 0.41 \\

         $^\dagger$Poisson rec. &
         0.72 & \topvalue{0.36} &
         0.59 & \runnerupvalue{0.20} &
         0.58 & \runnerupvalue{0.08} &
         \runnerupvalue{0.53} & \topvalue{0.09} &
         0.54 & \topvalue{0.10} &
         \runnerupvalue{0.67} & \topvalue{0.11} &
         \runnerupvalue{0.49} & \topvalue{0.07} &
         0.63 & \topvalue{0.08} &
         \runnerupvalue{0.71} & \topvalue{0.10} &
         \runnerupvalue{0.48} & \topvalue{0.06} \\

         \midrule
         EISR (Ours) &
         \topvalue{0.51} & 0.57 &
         \topvalue{0.31} & 0.54 &
         \topvalue{0.29} & 0.22 &
         \topvalue{0.32} & 0.52 &
         \topvalue{0.45} & 0.80 &
         \topvalue{0.42} & 1.72 &
         \topvalue{0.46} & 0.54 &
         \topvalue{0.50} & 0.50 &
         \topvalue{0.47} & 0.69 &
         \topvalue{0.33} & \runnerupvalue{0.36} \\
         \bottomrule
    \end{tabular}
\end{table*}

\begin{figure*}
\centering
\includegraphics[width=1.0\textwidth]{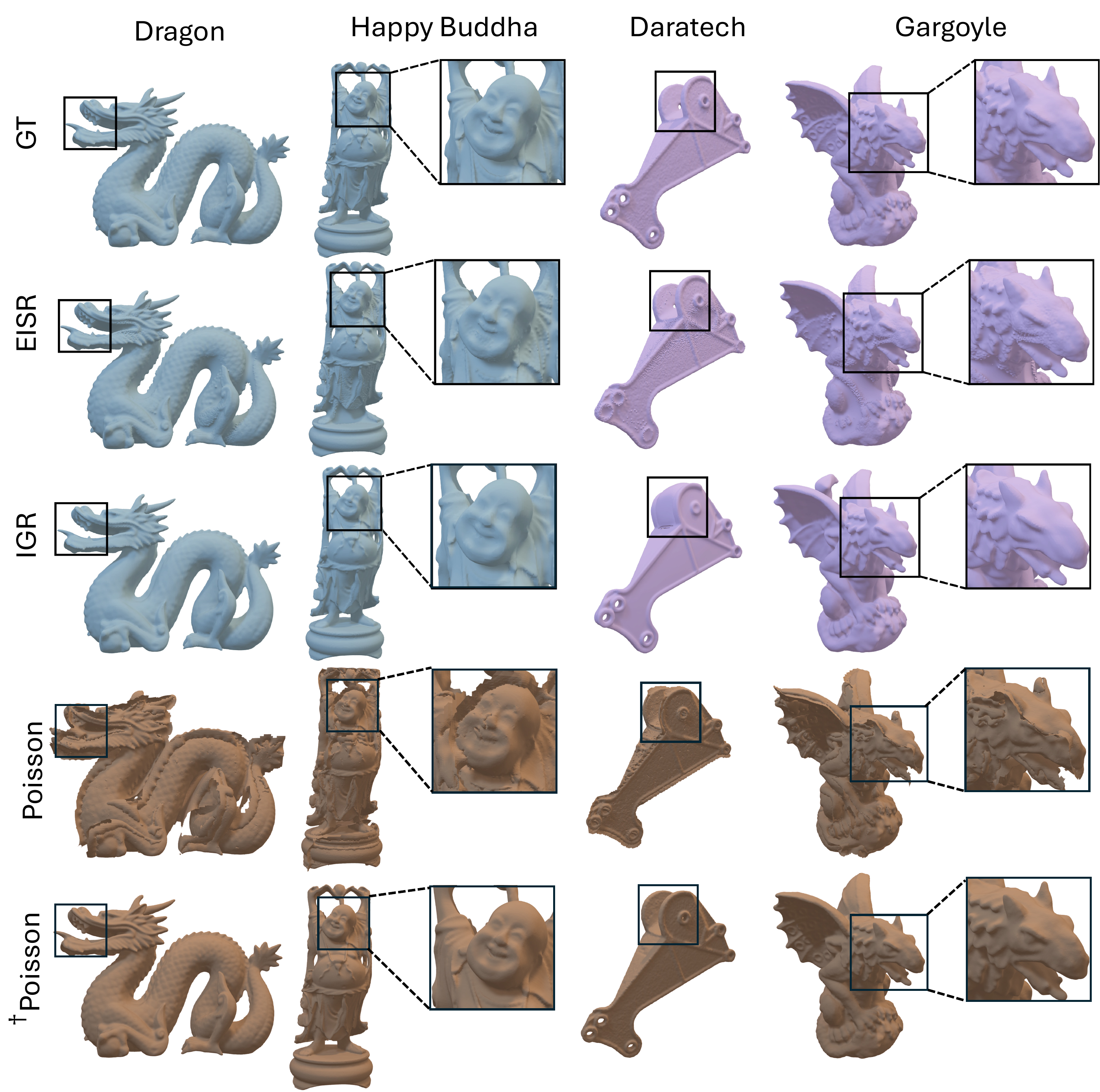}
\caption{Qualitative results on the Stanford 3D Scanning repository and the Williams et al. \citeyearpar{Williams2018DeepGP} dataset employed in IGR ~\citep{Gropp2020ImplicitGR}. Notice how our method effectively captures fine-resolution details from the target shapes.}
\label{fig:results_stanford_3d} 
\end{figure*}

Across both the Stanford 3D Repository and Williams et al. datasets, EISR achieves competitive $d_C$ and $d_H$, surpassing IGR and even Poisson reconstruction with full-oracle surface normals in many cases. While Poisson with ground-truth normals performs well in some cases, it requires privileged information. In contrast, EISR achieves competitive results using only the input point cloud. Minor increases in Hausdorff distance reflect small surface deviations, particularly in regions of high geometric complexity. However, the overall accuracy demonstrates that our Gaussian-based representation provides a robust, interpretable, and compact encoding of 3D shapes.

\subsection{Sensitivity to the Charge Initialization}

\begin{figure*}
    \centering
    \includegraphics[trim={0in 10in 0in 0in},clip,width=0.8\textwidth]{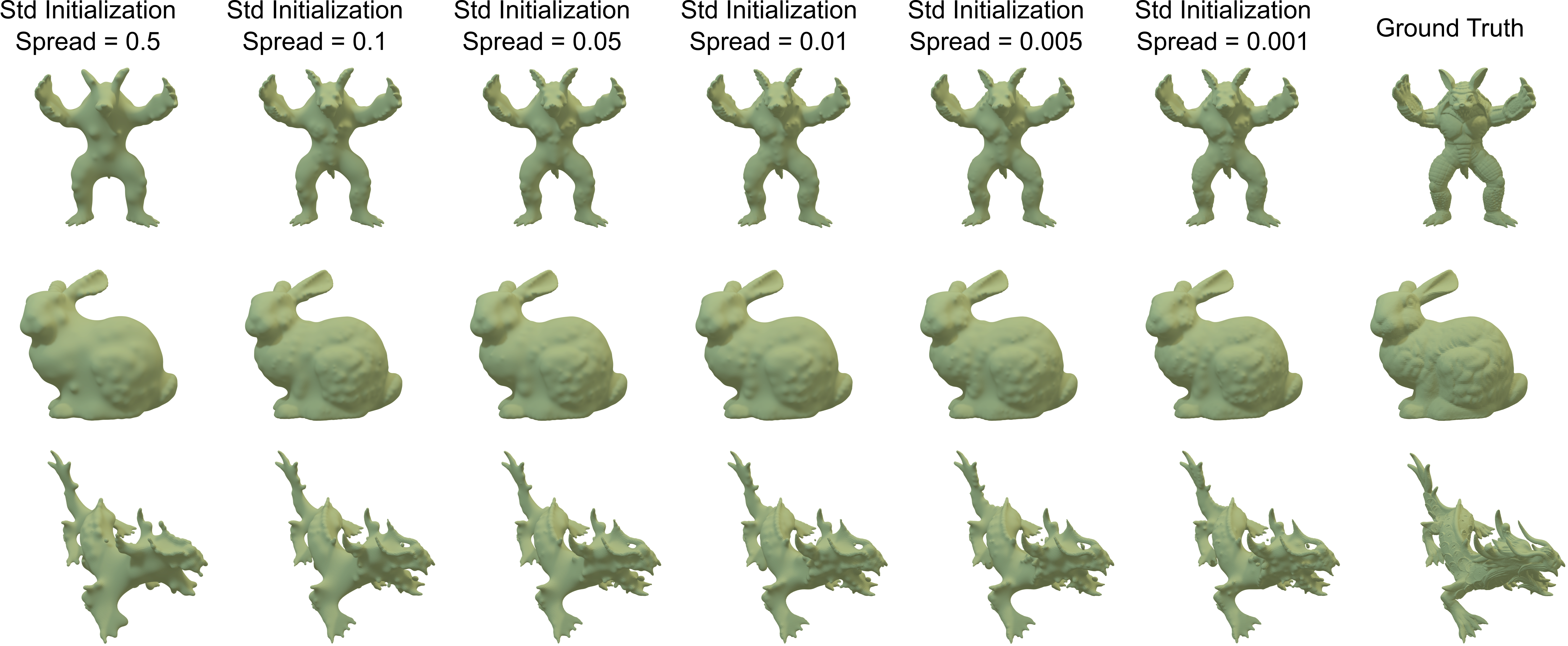}
    \caption{Reconstruction accuracy at different initializations of the charge density spread. We initialize the charges with a normal distribution centered at zero and vary their standard deviation.}
    \label{fig:vayring_initialization}
    \vspace{-15pt}
\end{figure*}
In this experiment, we investigate the behavior of the Gaussian charges in our approach. We vary their initialization and analyze the distribution of their locations after training. Changing the initial charge magnitudes or their initial locations had no noticeable effect after optimization. However, we noticed a significant effect on the quality of the recovered surfaces when we varied the standard deviation of the normal distribution we used to initialize the charge spread. When we allow smaller charge spreads at initialization, the reconstruction favors high-frequency details. In contrast, restricting initialization to larger charges yields an EISR approximation that resembles the target shape overall but lacks detail. We summarize these results in Fig. \ref{fig:vayring_initialization}.

We theorize these results arise from the individual contribution of the charge densities to the energy levels of the Fourier Transform of $\potential(\vx)$. By applying the convolution theorem to Eq. \ref{eq:conv_solution}, and following Eq. \ref{eq:shape_prior_solution}, it is possible to derive an expression for the Fourier Transform and its magnitude such that
\begin{equation}
    | \FT{\potential_{\objectdomain}}(\boldsymbol{\omega}) | = \frac{Q}{\sigma^2\pi \| \boldsymbol{\omega} \|^{2}} \exp^{-2 \sigma^2 \pi^2 \| \boldsymbol{\omega} \| ^2}.
    \label{eq:fourier_transform_of_the_potential}
\end{equation}
Note how in Eq. \ref{eq:fourier_transform_of_the_potential}, smaller charges contribute more to the energy of high frequencies. In comparison, larger charges contribute to the low frequencies. We offer an extended derivation of the Fourier Transform of $\potential_{\objectdomain}(\vx)$ in Appx. \ref{appx:fourier}. Additionally, in Appx.~\ref{appx:varying_charges_vis}, we discuss and visualize charges with varying magnitudes/spreads, highlighting how they organize when approximating the implicit field.

\begin{figure}
    \centering
    \includegraphics[width=0.9\linewidth]{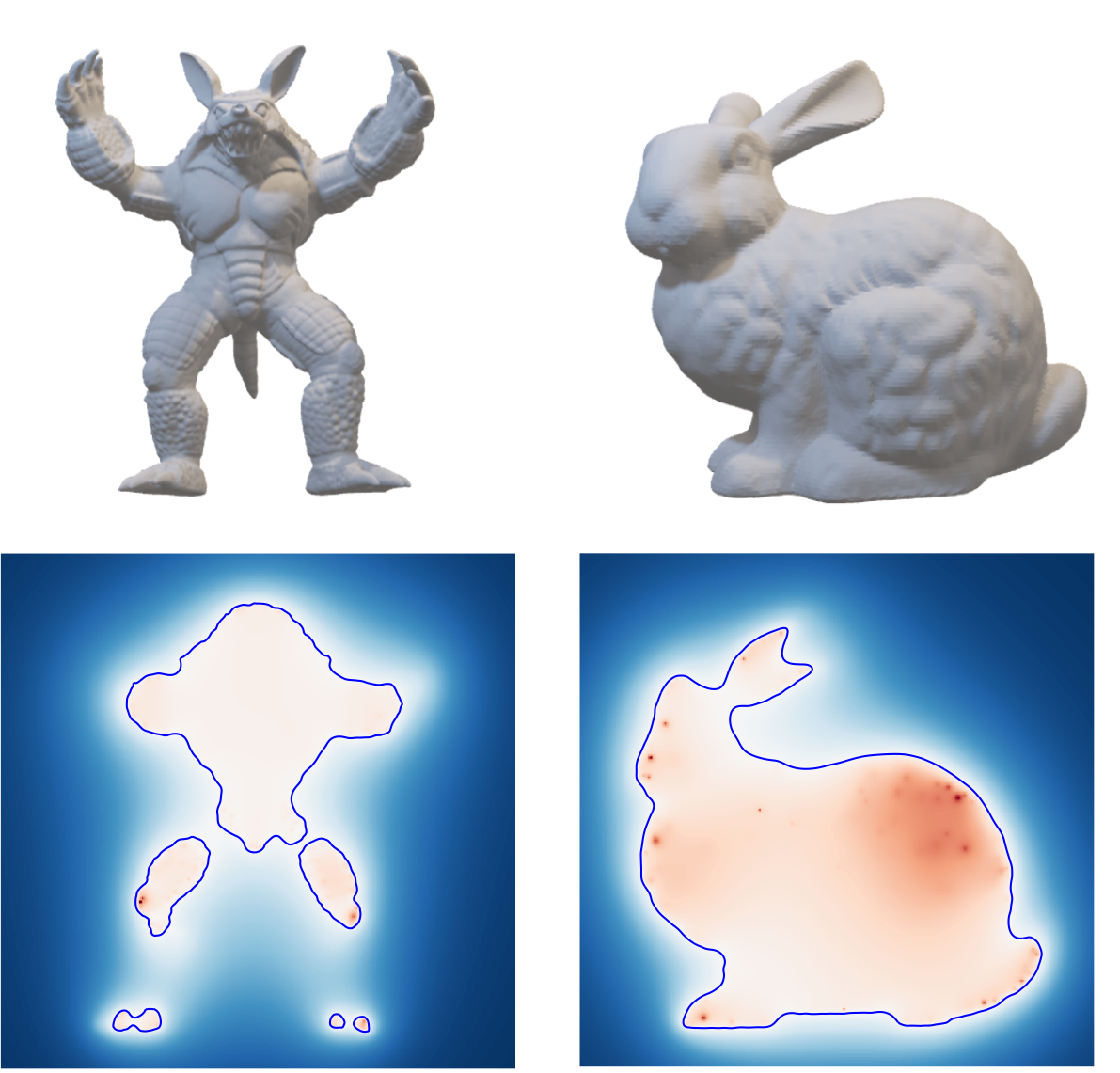}
    \caption{Cross-section of scalar fields computed with our method. Note how the zero-level set aligns consistently with the target surface.}
    \label{fig:cross_section}
    \vspace{-15pt}
\end{figure}

\section{Limitations}

We have identified some limitations in our approach that we will explore in future work. First, we use full supervision to optimize the charges' parameters. The choice of supervision limits the application of our approach when only partial views or RGB images are available. However, one can extend our method to predict color and thus use photometric loss functions to train with more realistic supervision, similar to Neural Rendering approaches. Another limitation we observed is that flat and concave surfaces are challenging for our method. With our Gaussian shape priors, it would require an infinite number of charges to reconstruct a shape \textit{perfectly}. However, EISR achieves reasonably good approximations by increasing the number of charges. Some potential future work is exploring different shape priors with approximation power for these types of surfaces by finding simplifying expressions for Eq. \ref{eq:potential_solution}, e.g., charge lines or planes. Finally, our method's efficiency bottleneck is the need to compute pairwise distances between the query point $\vx$ and each charge distribution's location $\vs_{i}$. In future work, we will explore alternatives to reduce computational cost and make EISR more efficient.



\section{Conclusions}

We introduced a new approach for shape representation, called Electrostatics-Inspired Surface Reconstruction (EISR), as an alternative to SDF-based surface reconstructions. Inspired by the extensive use of Poisson's PDE in computer vision and by its connection to electrostatics, we adopt it as a substitute for the Eikonal equation-based methods and solve the PDE analytically using Green's functions. Consequently, we found a closed-form expression for the PDE's solution in terms of the charge distribution, which we parameterize to recover the surface of the desired shape. Our approach demonstrated its ability to capture shape details as the number of charges increased, while maintaining a relatively low number of charges compared to the ground-truth mesh resolution. 




\bibliography{references}
\bibliographystyle{icml2026}

\newpage
\appendix
\onecolumn

\section{The Minimum Principle}\label{appx:min_principle}

This section provides the definition and proof of the \textit{minimum principle}. \textit{The Minimum principle} allows us to justify using Poisson's equation as an alternative to the Eikonal equation. It guarantees that under the conditions $\chargedensity(\vx) > 0 \; \forall \vx \in \sR^{3}$, the value of $\potential(\vx)$ inside a region a region $\objectdomain$ will be always greater or equal than the value of $\potential(\vx)$ on its boundary $\objectboundary$.

\begin{proposition}[Minimum principle] Let $\objectdomain$ be a closed region with boundary $\objectboundary$, and let $f(\vx)$ be a solution to the PDE 
\begin{equation}
\laplacian f(\vx) = g(\vx), \; \vx \in \sR^{3}.\label{eq:pde_min_principle}
\end{equation}
The \textit{minimum principle} states that, if $g(\vx) < 0\; \forall \vx \in \Omega$, then $f(\vx)$ takes its minimum value somewhere on $\objectboundary$ and not in the interior of the region\label{prop:min_principle}.

\end{proposition}

\begin{proof} We will proof Proposition \ref{prop:min_principle} by contradiction.
Let us assume that $f(\vx)$ is a solution to Eq. \ref{eq:pde_min_principle}. Additionally, suppose now that $f(\vx)$ reaches a minimum value $M = f(\vx^{*})$ for some $\vx^{*}$ inside $\Omega$. According to the second derivative test, $\vx^{*}$ is a minimum point of $f(\vx)$ if and only if its Hessian matrix $\mH(f)(\vx)$ is positive-definite when evaluated at $\vx^{*}$. If $\mH(f)(\vx^{*})$ is positive definite, then it has a Cholesky decomposition such that 
$$\mH(f)(\vx^{*}) = \mL\mL^{\intercal},$$ 
with $\mL$ a lower triangular matrix. Thus, each element in the main diagonal of $\mH(f)(\vx^{*})$ is strictly positive of the form 
$$\mH_{k,k} = \sum_{j=1}^{k} \mL_{k,j}^{2}.$$ 
This result implies that the trace of $\mH(f)(\vx^{*})$ is also strictly positive and then
$$
0 < Tr(\mH(f)(\vx^{*})) = \Delta f(\vx^{*}) = g(\vx^{*}) < 0,
$$
which is a contradiction for any point $\vx^{*} \in \Omega$. Consequently, $f(\vx)$ can only reach a minimum on $\objectboundary$. 


\end{proof}

To apply the \textit{minimum principle} to our electrostatic-inspired surface reconstruction method, we require all the charges to be positive such that the source function in Eq. \ref{eq:poisson_eq} is
$$ - \frac{\chargedensity(\vx)}{\epsilon_{0}} < 0\; \forall \vx \in \sR^{3}.$$
Therefore, our electrostatic potential takes its minimum value at each point $\vx$ on the boundary of the region of interest since $\objectboundary$ is a level-set of $\potential(\vx)$. Note that a similar analysis is possible for any level set.

\section{Full Experiments on Varying the Number of Charges}\label{apx:varying_gaussians}

Surface reconstruction with different numbers of Gaussian charges (shape priors) on a subset of the Stanford 3D Scanning repository. We present here the full table with the results, omitted in the main paper due to space restrictions. Our method shows significant results for all the metrics, even at a lower number of shape priors.

\begin{table*}[h]
    \centering
    \setlength{\tabcolsep}{1.7pt}
    \caption{Surface reconstruction with different numbers of Gaussian charges.}
    \label{tab:varying_num_charges_full}
    \scriptsize
    \begin{tabular}{lccccccccccccccccccc}
        \toprule
        \textbf{Model}  & \multicolumn{3}{c}{\textbf{1000 Charges}} & \multicolumn{3}{c}{\textbf{2000 Charges}} & \multicolumn{3}{c}{\textbf{5000 Charges}} & \multicolumn{3}{c}{\textbf{10000 Charges}} & \multicolumn{3}{c}{\textbf{20000 Charges}} & \multicolumn{3}{c}{\textbf{50000 Charges}} & \textbf{\# verts.}\\
        \cmidrule(lr){2-4}\cmidrule(lr){5-7}\cmidrule(lr){8-10}\cmidrule(lr){11-13}\cmidrule(lr){14-16}\cmidrule(lr){17-19}
         & F1($\uparrow$) & IoU($\uparrow$) & NC($\uparrow$) & F1($\uparrow$) & IoU($\uparrow$) & NC($\uparrow$) & F1($\uparrow$) & IoU($\uparrow$) & NC($\uparrow$) & F1($\uparrow$) & IoU($\uparrow$) & NC($\uparrow$) & F1($\uparrow$) & IoU($\uparrow$) & NC($\uparrow$) & F1($\uparrow$) & IoU($\uparrow$) & NC($\uparrow$) & \\
        \midrule
        Bunny & 97.123 & 0.978 & 0.992 & 97.703 & 0.979 & 0.994 & 97.964 & 0.982 & 0.995 & 98.090 & 0.984 & 0.995 & 98.075 & 0.984 & 0.996 & 98.036 & 0.984 & 0.995 & $\expnumber{3.6}{4}$ \\
        Dragon & 89.845 & 0.904 & 0.969 & 93.981 & 0.913 & 0.978 & 96.593 & 0.931 & 0.985 & 97.207 & 0.957 & 0.989 & 97.577 & 0.966 & 0.991 & 97.712 & 0.970 & 0.991 & $\expnumber{5.7}{5}$  \\
        Armadillo & 98.092 & 0.934 & 0.985 & 99.236 & 0.939 & 0.990 & 99.646 & 0.946 & 0.993 & 99.777 & 0.967 & 0.995 & 99.824 & 0.977 & 0.996 & 99.849 & 0.978 & 0.997 & $\expnumber{1.7}{5}$  \\
        Happy Buddha & 91.281 & 0.851 & 0.966 & 93.778 & 0.847 & 0.973 & 95.985 & 0.863 & 0.980 & 96.787 & 0.910 & 0.986 & 97.048 & 0.935 & 0.988 & 97.265 & 0.945 & 0.991 & $\expnumber{5.4}{5}$  \\
        Thai statue & 91.264 & 0.758 & 0.960 & 95.272 & 0.759 & 0.968 & 97.537 & 0.766 & 0.972 & 98.622 & 0.818 & 0.980 & 99.081 & 0.857 & 0.985 & 99.194 & 0.875 & 0.988 & $\expnumber{3.6}{6}$  \\
        Asian dragon & 90.835 & 0.795 & 0.958 & 92.035 & 0.771 & 0.959 & 95.240 & 0.764 & 0.966 & 96.235 & 0.762 & 0.968 & 96.459 & 0.815 & 0.980 & 96.829 & 0.856 & 0.984 & $\expnumber{1.9}{7}$  \\
        \bottomrule
    \end{tabular}
\end{table*}

\section{Fourier Transform of the Parameterized Implicit Field}\label{appx:fourier}

One of our method's main results is to express the implicit scalar field we use to reconstruct the target shape, as a convolution of the Green's function $G(\vx - \vs)$ and the charge density $\chargedensity(\vx)$. Thanks to this result and because Poisson's equation is a linear PDE, it is possible to derive an expression for the Fourier transform of $\potential_{\Omega}$.

First, recall that the convolution theorem allows us to express the potential function as
\begin{equation}
\FT{\potential_{\objectdomain}}(\boldsymbol{\omega}) = \FT{G \ast \chargedensity}(\boldsymbol{\omega}) =  \left[ \FT{G} \cdot \FT{\chargedensity} \right] (\boldsymbol{\omega}).\label{eq:conve_theorem}    
\end{equation}

Since we chose to parameterize $\potential(\vx)$ as a Gaussian charge density, its Fourier transform is also a Gaussian. We obtain the Fourier transform of the Green's function by using spectral differentiation such that
\begin{align}
\FT{\laplacian G}(\boldsymbol{\omega}) &= \FT{\delta}(\boldsymbol{\omega}) \nonumber \\
\| \boldsymbol{\omega} \|^{2} \FT{G}(\boldsymbol{\omega}) &= 1 \nonumber \\
\FT{G}(\boldsymbol{\omega}) &= \frac{1}{\| \boldsymbol{\omega} \|^{2}}\label{eq:fourier_transform_of_G} 
\end{align}
Computing the element-wise product of Eq. \ref{eq:conve_theorem} and \ref{eq:fourier_transform_of_G} for a single Gaussian charge density and taking  its magnitude yields the results shown in Eq. \ref{eq:fourier_transform_of_the_potential}:
$$| \FT{\potential}(\boldsymbol{\omega}) | = \frac{Q}{\sigma^2\pi \| \boldsymbol{\omega} \|^{2}} \exp^{-2 \sigma^2 \pi^2 \| \boldsymbol{\omega} \| ^2}.$$

This result shows how charges with small spreads denoted with parameter $\sigma$ contribute more to the energy of high-frequencies in $\FT{\potential}(\boldsymbol{\omega})$. In contrast, charges with a larger spread contribute to the low-frequencies.




\section{Visualization of Charge Distribution}\label{appx:varying_charges_vis}

In this section, we discuss and visualize charges with varying magnitudes/spreads. We show the location of the charges, $\vs_{i}$, after the optimization is complete. We highlight how they organize during the approximation of the implicit field and present results for progressively smaller charge magnitudes and charge spreads. We summarize these results in Fig. \ref{fig:charge_distribution}. Note how the charges naturally accumulate near the surface, concentrating on areas where finer details are needed. Similarly, note how larger charge spreads end up in areas of the target shape that can be approximated with low frequencies, e.g., the torso and the legs of the armadillo model.

\begin{figure*}
\centering
    \setlength\belowcaptionskip{-1.0\baselineskip}
     \centering     
     \includegraphics[width=1.0\textwidth, trim=12.0cm 0.0cm 12.0cm 0.0cm,clip]{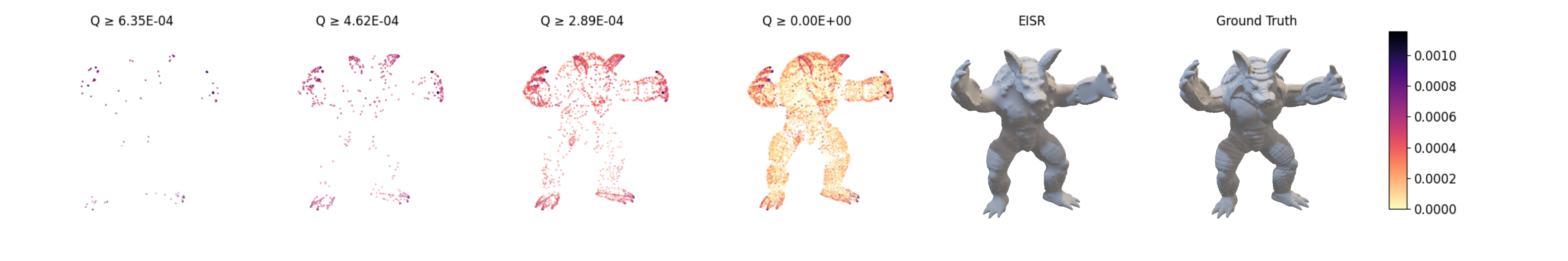}
    
    \small{a) Charge locations $Q_{i}$ after optimization, sorted by their magnitude.}   
    
     \includegraphics[width=1.0\textwidth, trim=12.0cm 0.0cm 12.0cm 0.0cm,clip]{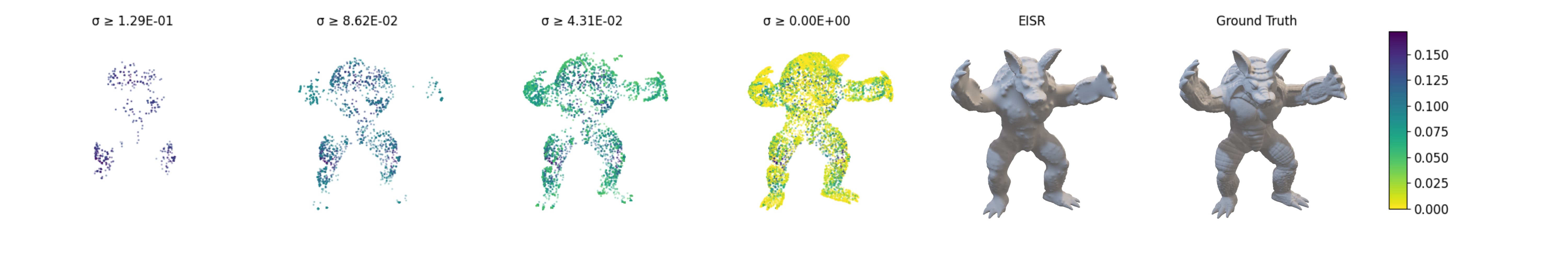}     
    
     \small{b) Charge spread $\sigma_{i}$ after optimization sorted by their magnitude.}           
     
    \caption{Final distribution of the charge's locations after optimization. Note how a) the charges tend to organize themselves near the surface and b) larger charge spreads are placed in areas where low-frequency details are needed.}\label{fig:charge_distribution}
\end{figure*}

\section{Additional Results}

We provide additional results to the experiments in the paper to further demonstrate the properties of our shape representation method.

\begin{figure*}
    \centering
    \includegraphics[height=0.95\textheight]{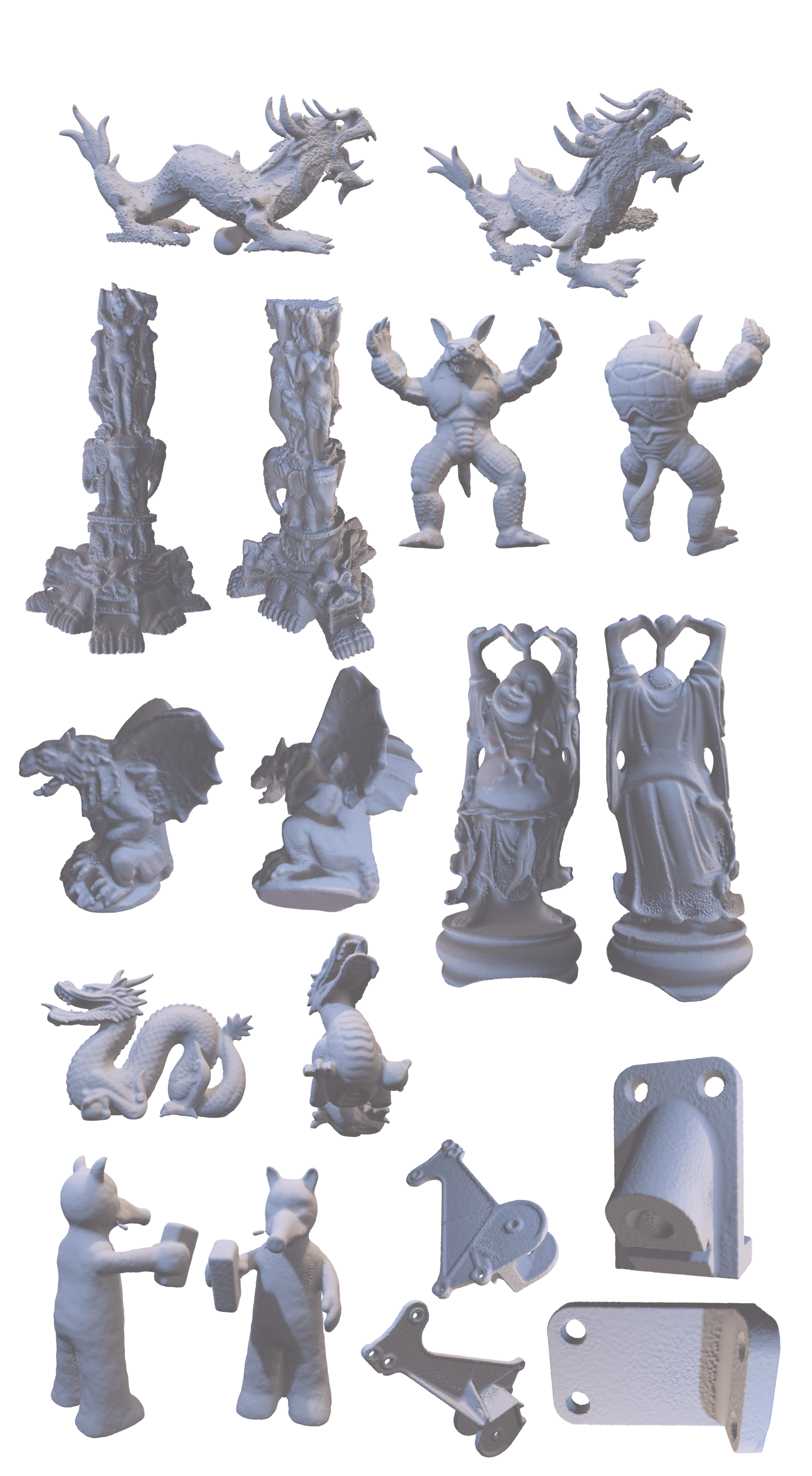}
    \caption{Surface Reconstruction additional results.}
    \label{fig:shapenet_objects}
\end{figure*}

We have attached five of the reconstructed meshes in \textit{*.obj} format for a better observation of the reconstruction accuracy. The reconstruction can be compared against the ground truth meshes from the Stanford 3D Scanning repository and the Williams et al. dataset.  The ground truth meshes need to be scaled to a unit cube for comparison. See attached files.





\end{document}